\documentclass[12pt]{article}

\usepackage{newtxtext,newtxmath}

\usepackage{graphicx}

\usepackage[letterpaper,margin=1in]{geometry}

\renewenvironment{abstract}
	{\quotation}
	{\endquotation}

\date{}

\makeatletter
\renewcommand{\fnum@figure}{\textbf{Figure \thefigure}}
\renewcommand{\fnum@table}{\textbf{Table \thetable}}
\makeatother

\usepackage{scicite}

\usepackage{url}

\usepackage{amsmath,amsthm}
\newtheorem{proposition}{Proposition}

\usepackage{nomencl}

\usepackage{algorithm}
\usepackage{algpseudocode}
\usepackage{siunitx}
\usepackage[colorlinks=true, linkcolor=blue, citecolor=blue, urlcolor=blue]{hyperref}

\def\scititle{
	V*: An Efficient Motion Planning Algorithm for Autonomous Vehicles
}
\title{\bfseries \boldmath \scititle}

\author{
	Abdullah~Zareh~Andaryan$^{1\ast}$,
	Michael~G.H.~Bell$^{1}$,
	Mohsen~Ramezani$^{2}$,
        Glenn~Geers$^{1}$\and
	\small$^{1}$Institute of Transport and Logistics Studies, Business School, the University of Sydney, Sydney, 2006, NSW, Australia\and
	\small$^{2}$School of Civil Engineering, Faculty of Engineering, the University of Sydney, Sydney, 2006, NSW, Australia\and
	\small$^\ast$Corresponding author. Email: abdullah.zareh@sydney.edu.au
}

\begin{document} 

\maketitle

\begin{abstract} \bfseries \boldmath
Autonomous vehicle navigation in structured environments requires planners capable of generating time-optimal, collision-free trajectories that satisfy dynamic and kinematic constraints. We introduce V*, a graph-based motion planner that represents speed and direction as explicit state variables within a discretised space–time–velocity lattice. Unlike traditional methods that decouple spatial search from dynamic feasibility or rely on post-hoc smoothing, V* integrates both motion dimensions directly into graph construction through dynamic graph generation during search expansion.

To manage the complexity of high-dimensional search, we employ a hexagonal discretisation strategy and provide formal mathematical proofs establishing optimal waypoint spacing and minimal node redundancy under constrained heading transitions for velocity-aware motion planning. We develop a mathematical formulation for transient steering dynamics in the kinematic bicycle model, modelling steering angle convergence with exponential behaviour, and deriving the relationship for convergence rate parameters. This theoretical foundation, combined with geometric pruning strategies that eliminate expansions leading to infeasible steering configurations, enables V* to evaluate dynamically admissible manoeuvres ensuring each trajectory is physically realisable without further refinement. We further demonstrate V*'s performance in simulation studies with cluttered and dynamic environments involving moving obstacles, showing its ability to avoid conflicts, yield proactively, and generate safe, efficient trajectories with temporal reasoning capabilities for waiting behaviours and dynamic coordination.

\end{abstract}

\noindent
Autonomous navigation requires planning algorithms that can compute time-efficient and dynamically feasible trajectories. Among the diverse approaches to motion planning, optimal pathfinding algorithms are recognized for their ability to guarantee high-quality solutions by minimizing path costs under strict constraints. This makes them particularly valuable in environments where precision and efficiency are critical. Classical pathfinding methods such as the Dijkstra~\cite{dijkstra1959} and A*~\cite{hart1968} algorithms have been widely adopted for solving shortest path problems on graphs due to their efficiency. However, their applications are typically limited in physically grounded motion planning domains. These methods typically assume holonomic motion and disregard vehicle-specific limitations such as steering curvature, acceleration bounds, and directional inertia. In contrast, real-world autonomous vehicles \footnote{A vehicle equipped with an ”automated driving system” capable of performing on-road operations either without occupants or with passengers who are not drivers or fallback-ready users \cite{sae2021}} operate under non-holonomic constraints (it cannot move freely in all directions because the motion depends on both position and velocity) and dynamic limitations, such as restricted turning radii, bounded accelerations, and fixed heading directions.  These characteristics pose significant challenges for traditional grid-based pathfinding methods.

 To address these limitations while preserving the advantages of graph-based search approaches, this paper presents V* algorithm, a motion planning framework that constructs a discretised space–time–velocity lattice, in which each node captures not only spatial location and heading but also forward speed. By integrating speed and direction into the graph itself, V* allows the planner to evaluate transitions that are dynamically admissible from the outset. Motion feasibility is enforced via a kinematic model, ensuring that generated paths inherently respect turning radii and acceleration bounds without the need for downstream correction.

The design of the search space is grounded in a hexagonal discretisation, chosen not only for its geometric regularity but also for its provably optimal properties under constrained motion. By analytically minimizing redundancy in node placement, this configuration balances spatial resolution with computational efficiency.  Furthermore, a geometric pruning mechanism is developed to restrict node expansion based on steering and turning feasibility, incorporating our theoretical model of steering angle convergence behavior during vehicle maneuvering. 
The algorithm provides theoretical guarantees on solution quality and completion: we prove that V* finds optimal paths within the discretised lattice and always terminates when a solution exists.

The scope of this work is limited to the single-agent case, which serves as the basis for future extensions to multi-agent scenarios. The paper formalizes the structure of the V* algorithm, establishes optimality conditions under consistent and admissible heuristics, and introduces a geometric framework to control the density of node expansion. Simulations illustrate the algorithm's effectiveness in generating realistic, kinematically admissible paths in complex environments, demonstrating its potential as a foundational approach for dynamic vehicle motion
planning. 

\subsection*{Related works}

The A* algorithm, introduced by Hart, Nilsson and Raphael in 1968~\cite{hart1968}, is a simple yet highly effective graph-based method widely utilized in optimal pathfinding within artificial intelligence and autonomous navigation ~\cite{wang2024}. It evaluates potential paths using a combined cost function
\begin{equation}
 f(n) = g(n) +\widehat{h}(n)
\end{equation}
where \(g(n)\) is the accumulated cost from the start node to the node \(n\), and \(\widehat{h}(n)\) is a heuristic estimate of the cost from \(n\) to the goal. If the heuristic is admissible and consistent, A* guarantees an optimal path~\cite{hart1968}. Despite its effectiveness, A*'s straightforward assumptions can become overly restrictive in complex and dynamic scenarios, particularly involving non-holonomic vehicles. Traditional grid-based graph representations underlying A* are typically inadequate for capturing essential kinodynamic constraints, such as curvature and acceleration, necessary for accurate path planning in realistic vehicle navigation contexts ~\cite{yu2024}.

This recognition gave rise to kinodynamic extensions of A* that embed physical constraints such as acceleration, curvature, and turning radius directly into the planning model. The most notable advancement is Hybrid A* \cite{pivtoraiko2009differentially}, which represents the closest precedent to velocity-aware graph search. Hybrid A* incorporates kinodynamic constraints through precomputed motion primitives that connect discrete states defined by position and orientation ($x, y, \theta$). These primitives embed velocity information within the trajectory segments (edges) rather than as explicit state parameters, enabling structured kinodynamic planning while maintaining A*'s optimality guarantees.

Building on this foundation, several variants have emerged to address specific limitations. Kinodynamic A* \cite{likhachev2009planning} expands states that respect acceleration constraints but relies on fixed primitive libraries. McNaughton et al. (2011) developed spatiotemporal lattices for highway scenarios, though with fixed-speed assumptions that limit adaptability to variable velocity requirements. More recently, Dehghani Tezerjani et al. (2024) proposed a hierarchical approach combining enhanced A* with Time Elastic Band (TEB) local planning, achieving real-time performance through adaptive waypoint density control, though velocity optimization remains decoupled from the graph search phase.

Alternative graph-based approaches have explored different representations. Marcucci et al.~\cite{marcucci2023} introduce "graphs of convex sets" where regions represent collision-free space with admissible speed ranges under kinematic bicycle models \cite{polack2017kinematic}. Regarding discretisation strategies, \cite{an2024improved} advocate for hexagonal grid structures, citing their equal-directional movement cost and reduced directional bias. Their empirical findings validate smoother trajectories and reduced turning penalties. However, these studies largely assess geometric or computational metrics and do not formally integrate vehicle motion constraints into the analysis.

Despite these advances, existing A*-based kinodynamic planners share fundamental limitations: velocity is either embedded in motion primitives (edges), handled through post-processing, or treated as an optimization variable separate from graph search. Additionally, these approaches rely on offline preprocessing to generate complete libraries of feasible motions before planning begins, constraining the planner's expressiveness to predetermined transitions and limiting adaptability to novel scenarios not anticipated during preprocessing \cite{likhachev2009planning, pivtoraiko2009differentially}. No current approach successfully combines explicit velocity encoding as a structural parameter of graph nodes with dynamic graph generation.

The V* algorithm addresses this gap by constructing a fully discrete space-time-velocity lattice where each node explicitly encodes position, heading, forward speed, and time ($x$, $y$, $\theta$, $v$, $t$), while simultaneously generating the graph dynamically during search expansion. This represents a fundamental architectural departure from existing kinodynamic A* approaches across multiple dimensions. In Hybrid A*, velocity is implicitly embedded within motion primitives—precomputed trajectory segments that connect discrete configuration states ($x$, $y$, $\theta$). These primitives contain velocity profiles as properties of the edges, making velocity a characteristic of transitions rather than states. Likhachev and Ferguson's approach \cite{likhachev2009planning} follows a similar paradigm, where velocity is derived from the execution of motion primitives that satisfy differential constraints. While these approaches successfully incorporate kinodynamic constraints, they require the planner to reason about velocity indirectly through primitive selection and composition.

The second critical difference lies in graph construction methodology. Hybrid A* relies on offline computation of motion primitive libraries using boundary value problem solvers to connect lattice nodes while satisfying kinodynamic constraints. These primitives, typically generated using techniques like Reeds-Shepp curves or clothoid segments, form a fixed repertoire of feasible motions that determines the planner's expressiveness. The resolution and diversity of solutions are fundamentally constrained by the primitive library's comprehensiveness—maneuvers not represented in the precomputed set cannot be discovered during search. Kinodynamic A* faces similar limitations, requiring extensive offline computation to generate sufficient primitive diversity while maintaining real-time performance. V* eliminates this dependency through dynamic feasibility verification using geometric pruning and real-time kinematic bicycle model checking. Rather than consulting a primitive database, V* generates feasible transitions on-demand during node expansion, enabling discovery of motion combinations not explicitly precomputed. This approach trades offline computational burden for online flexibility, allowing adaptation to novel scenarios without primitive regeneration.

Heuristic design and optimality considerations further distinguish these approaches. The heuristic functions in Hybrid A* typically rely on geometric approximations such as Euclidean distance, Dubins path lengths, or Reeds-Shepp distances from the current configuration to the goal. These heuristics, while admissible, often significantly underestimate the true cost in complex kinodynamic scenarios, leading to extensive search expansion. Kinodynamic A* faces similar challenges in designing informative heuristics that account for velocity and acceleration constraints. V* enables a new class of velocity-aware heuristics specifically designed for the space-time-velocity lattice. The proposed heuristic accounts for both spatial and temporal aspects of motion, providing more accurate cost-to-go estimates that incorporate velocity-dependent travel times. This improved heuristic guidance, combined with the explicit velocity representation, enables more efficient search convergence while maintaining optimality guarantees.

Temporal reasoning capabilities differ substantially between approaches. Traditional motion primitive approaches typically focus on kinematic maneuvers and active motion scenarios, with their foundational papers \cite{likhachev2009planning} concentrating on high-speed driving and parking navigation without explicitly modeling stationary states within their core frameworks. The configuration-based state representation ($x$, $y$, $\theta$) makes it challenging to directly incorporate time-dependent behaviors such as intentional waiting, temporal coordination with moving obstacles, or hold-in-place actions, often requiring these considerations to be handled through external coordination mechanisms, specialized primitive libraries, or replanning strategies.

V* provides a fundamental architectural advantage through its native time-velocity representation. By naturally incorporating zero-velocity states as explicit nodes in its state lattice naturally incorporates zero-velocity states as explicit nodes in its state lattice ($x$, $y$, $\theta$, $v=0$, $t$), enabling direct representation of stationary behaviors as valid planning outcomes. This capability allows the planner to reason about hold-in-place actions, temporal coordination with moving obstacles, and intentional waiting behaviors as integral parts of the motion plan, leveraging its native time-velocity representation to handle temporal dependencies within the search process itself.

The computational characteristics of these approaches reflect their different design philosophies. Hybrid A* operates in 3D configuration space ($x$, $y$, $\theta$) using precomputed motion primitives, while V* searches in a higher-dimensional space ($x$, $y$, $\theta$, $v$, $t$) with dynamic feasibility checking. This dimensional difference represents a fundamental trade-off: V* enables more direct velocity reasoning and temporal coordination at the cost of increased per-node computational complexity. We demonstrate that hexagonal discretisation and geometric pruning strategies help manage this computational overhead while preserving optimality guarantees within the discretised space.

All approaches provide optimality guarantees only within their discretised representations, facing the fundamental trade-off between computational tractability and continuous-space optimality. However, V*'s systematic discretisation of velocity and time dimensions enables fundamentally new planning capabilities that emerge naturally from its representation: temporal coordination with moving obstacles, intentional waiting behaviors, and direct velocity-time trade-off optimization. While primitive-based approaches could incorporate such behaviors through specialized design, V*'s native space-time-velocity representation makes these capabilities architecturally inherent rather than requiring additional algorithmic complexity.

\section*{Simulation Results}
\subsection*{Heuristic impact on V* efficiency}

The performance of the V* algorithm is highly dependent on the choice of
heuristic function. In the absence of a heuristic term, that is, when
\(h= 0\), the algorithm performs an exhaustive search by
expanding nodes purely based on their accumulated cost from the start.
While this guarantees optimality, it leads to unnecessary exploration of
nodes, significantly increasing the computational burden, especially in
large or complex environments.

A commonly used heuristic is the Euclidean distance between the current node and the goal. This heuristic is admissible and consistent, and in many open environments,
it offers a reasonable trade-off between exploration and exploitation.
However, its effectiveness diminishes in environments where the
straight-line path is not feasible due to obstacles or complex terrain.

In such cases, the straight-line heuristic may underestimate the actual
cost too severely, leading the algorithm to explore paths that appear
promising in terms of the heuristic but are practically blocked or
suboptimal. This misguidance can degrade performance, causing excessive
node expansions and increased run-time. 



An alternative to the straight-line heuristic is the waterflow heuristic
\cite{moore1959}, which aims to improve performance in complex environments,
where obstacles significantly constrain direct paths. Unlike the
Euclidean approach, which assumes a direct line to the goal,
the waterflow heuristic is derived from a simulated diffusion process
that mimics the behavior of water propagating from the goal location.

The algorithm begins by treating the goal as a source point and
propagates a cost field outward through the free space in the
environment. This is typically achieved using a wavefront expansion or a
discrete potential field, ensuring that cost values increase gradually
with distance from the goal while respecting obstacle boundaries. The
environment is discretised into cells, where each cell has a value based
on the minimum cost required to reach the goal, accounting for navigable
paths rather than just geometric distance.

Nevertheless, several drawbacks must be considered. High-resolution
grids can impose significant computational demands due to the extensive
propagation involved. Conversely, if the grid resolution is too coarse,
the propagated values may overestimate the true cost-to-go, thereby
violating admissibility. However, since the heuristic is computed only
once, assuming full knowledge of the environment, and the cost in V*
corresponds to the remaining travel time, a simple ratio of distance to
maximum speed already provides a reasonable underestimation. Therefore,
selecting a moderate grid resolution that balances computational effort
and heuristic accuracy is sufficient to ensure optimality. Fig.~\ref{heuristiccompare}
demonstrates the effect of using the waterflow algorithm in the node
expansion process.  Darker regions indicate proximity to the goal, with color intensity increasing as the distance decreases.

\begin{figure}[ht]
\centerline{\includegraphics[width=\columnwidth]{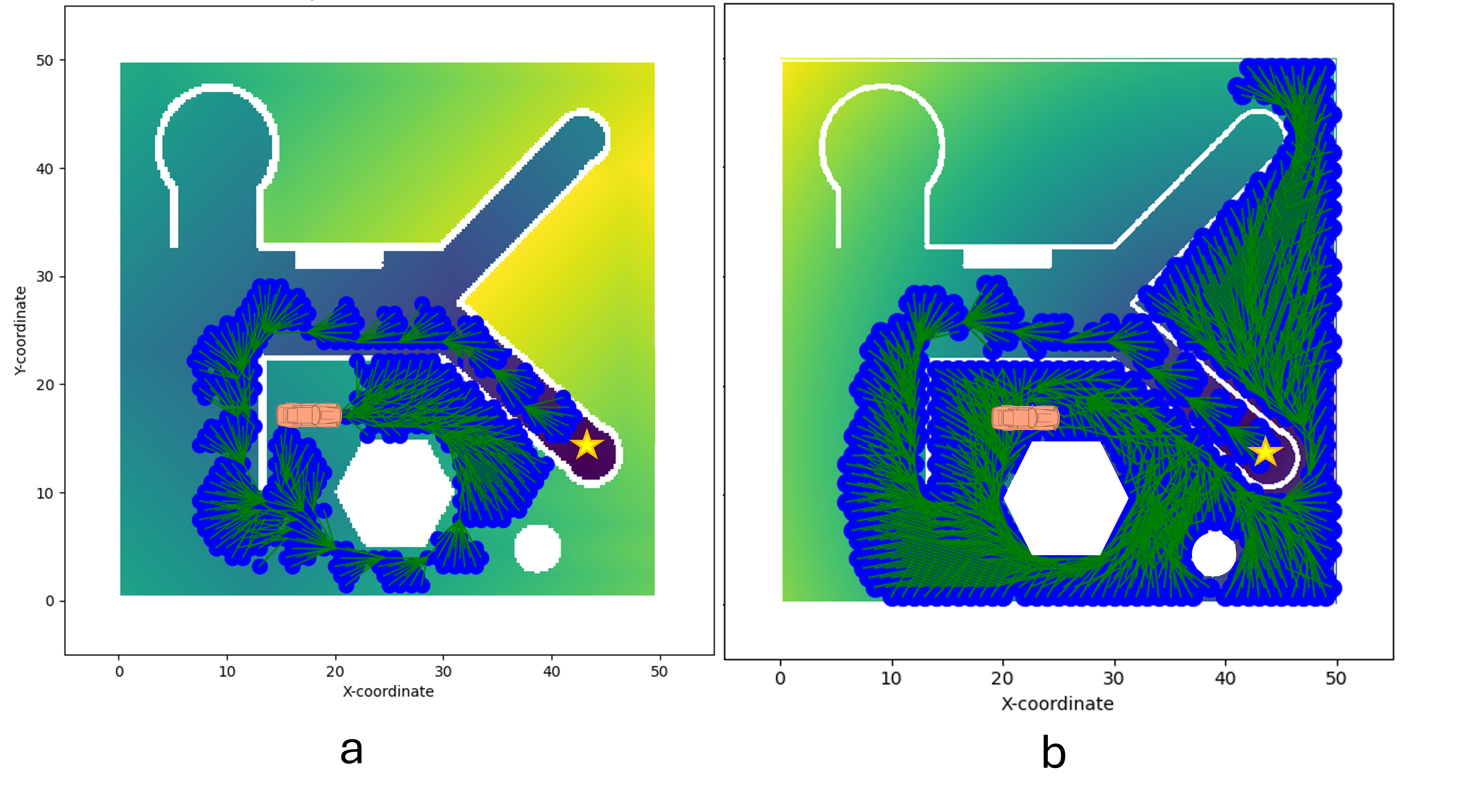}}
\caption{\textbf{Comparison of search expansions using the waterflow heuristic.}
The figure (a) demonstrates the effect of incorporating the waterflow heuristic, where the vehicle starts with a heading
opposite to the goal (star). In this case expansions are more
focused and efficiently directed toward the goal. The figure (b) shows the search tree generated with the Euclidean distance heuristic, resulting in wide exploration even in less promising
directions and regions. }
\label{heuristiccompare}
\end{figure}
To investigate the impact of heuristic functions, we compared the
performance of V* in generating trajectories under two settings: one
using a Euclidean distance heuristic and the other using a Waterflow-based heuristic. The map is implemented as an
environment that contains stationary obstacles and maze-shaped wall structures. The vehicle was initialised with zero speed and an arbitrary heading, and the goal region was defined as a bounded area
within a 1-meter tolerance from the target location. We assumed a
maximum speed of \qty[mode=text]{4}{\meter \per \second} and steering angles ranging from \ang{-30} to \ang{30},
with a wheelbase of \qty[mode=text]{2}{\meter}. Fig.~\ref{trajectory} shows the trajectory
generated by V* across four randomized trials. 
\begin{figure}[htbp]
\centerline{\includegraphics{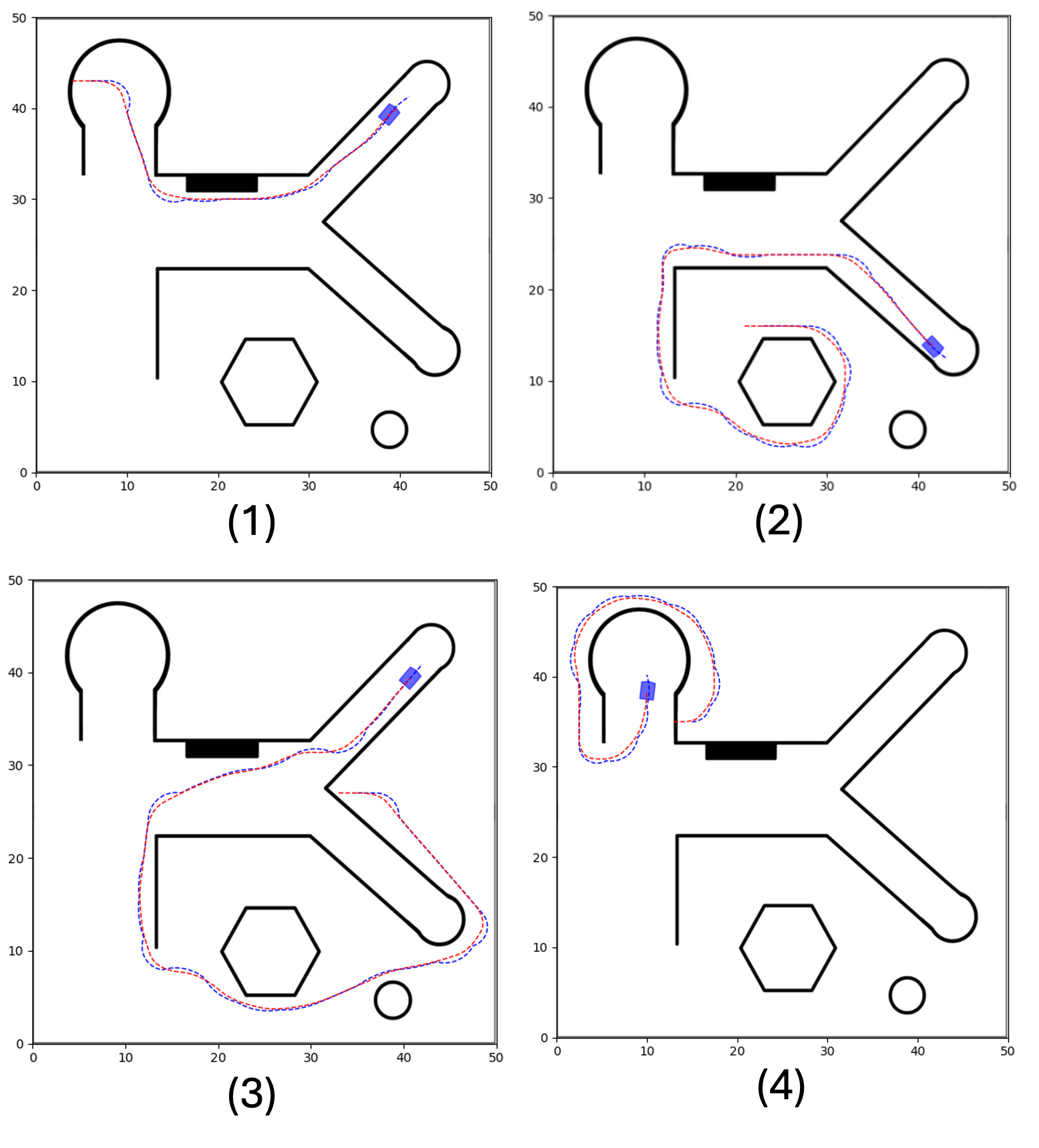}}
\caption{\textbf{Illustration of the planned trajectory generated by the V*
algorithm in four different trials.} The blue dashed line represents the
path of the front wheel, and the red line shows the trajectory of the
rear wheel. The final position of the AV is depicted as a blue pointed box.}
\label{trajectory}
\end{figure}

The vehicle successfully followed time-efficient paths that avoided obstacles and remained
kinematically feasible.Table 1 summarizes the number of expanded nodes,
CPU time, and total path cost across the four trials. The V* algorithm
consistently exhibited lower node expansion and runtime when guided by
the Waterflow heuristic. The Euclidean heuristic, by contrast, expanded
more nodes in non-promising directions due to its lack of environmental
awareness.
\begin{table}[htbp]
\caption{Comparison of Euclidean and Waterflow Heuristics}
\begin{center}
\begin{tabular}{|c|cc|cc|c|}
\hline
\textbf{Experiment} & \multicolumn{2}{c|}{\textbf{Euclidean}} & \multicolumn{2}{c|}{\textbf{Waterflow}} & \textbf{\begin{tabular}[c]{@{}c@{}}Path\\ Cost\\ (s)\end{tabular}} \\
\cline{2-5}
  & \textbf{\textit{Nodes}} & \textbf{\textit{Time(s)}} & \textbf{\textit{Nodes}} & \textbf{\textit{Time(s)}} & \\
\hline
1 & 194   & 0.617  & 173  & 0.095 & 12 \\
2 & 44889 & 19.818 & 839  & 0.216 & 22 \\
3 & 71619 & 31.385 & 529  & 0.144 & 26 \\
4 & 2066  & 1.077  & 281  & 0.122 & 14 \\
\hline

\end{tabular}\label{tab:heuristic_comparison}
\end{center}
\end{table}

\subsection*{Moving obstacles}
 We now evaluate its performance in dynamic scenarios involving moving obstacles. This evaluation shifts focus from heuristic comparison to temporal coordination capabilities, examining how V*'s native space-time-velocity representation handles the complexities of dynamic environments. Unlike static obstacle avoidance, planning with moving obstacles requires the algorithm to reason about future positions and potential conflicts, execute waiting behaviors when necessary, and adapt its trajectory timing to avoid collisions for safe navigation. To demonstrate these capabilities, we examine a representative scenario involving four moving obstacles presented in Fig.~\ref{fig:moving_objects_4}. The red agents follow predefined, fixed trajectories and are entirely unaware of the vehicle's presence. In contrast, the blue agent—operating under the V* algorithm—has full knowledge of the obstacles' future positions and velocities. Leveraging this information, the planner evaluates feasible trajectories in a discretised space–time–velocity lattice, allowing the vehicle to anticipate future conflicts and adjust its motion proactively.

Fig.~\ref{fig:moving_objects_4} presents a sequence of simulation frames showcasing the motion planning behavior of the AV (blue) navigating in an environment populated with multiple moving obstacles (red). 

This anticipatory behavior is particularly evident in the second frame, where the vehicle comes to a complete stop to yield to a high-speed obstacle crossing its intended path. Once the interfering obstacle has passed, the vehicle smoothly resumes its trajectory, adjusting speed and heading to safely navigate around other moving agents. 

\begin{figure}
    \centering
    \includegraphics[width=0.9\linewidth]{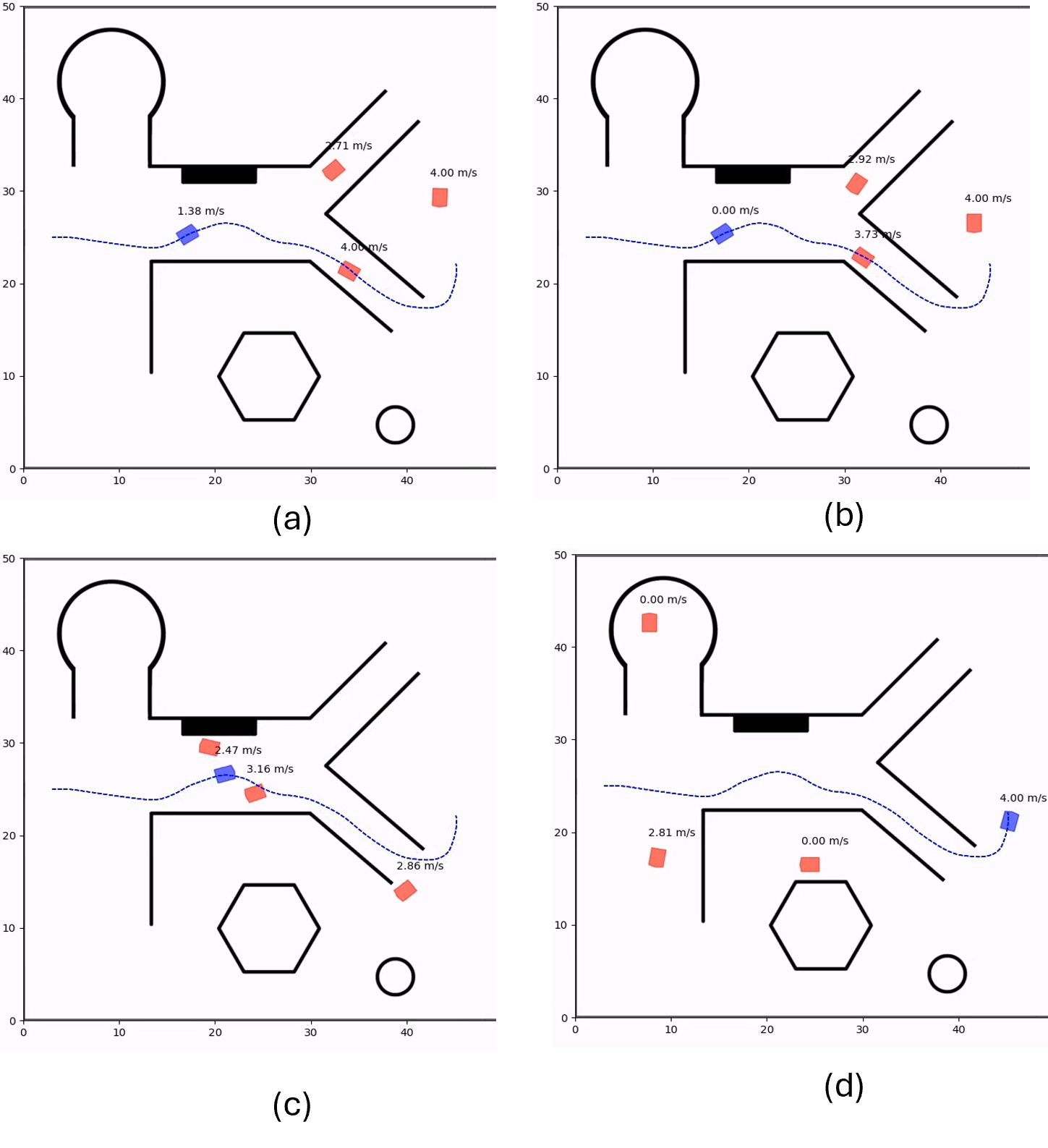}
    \caption{\textbf{Sequence of simulation snapshots illustrating the proactive behavior of the autonomous vehicle (blue) navigating through a structured environment populated with independent, non-cooperative moving obstacles (red)}. Obstacles move along fixed trajectories regardless of the vehicle’s motion. The blue dashed line represents the planned trajectory, and annotated speed labels (in m/s) show the instantaneous velocities of all agents.}
    \label{fig:moving_objects_4}
\end{figure}

\section*{Methodology}

We consider the problem of motion planning for a single AV tasked with navigating from a given start position to a specified goal in an environment populated with known obstacles  that may move over time, but whose initial positions and trajectories are assumed to be known and fixed during planning. The objective is to generate a trajectory that minimizes total travel time while ensuring kinematic feasibility and collision avoidance.

The vehicle is modeled as a non-holonomic system with a limited turning
radius, constrained to move in the forward direction. Its motion is
governed by a set of control parameters, including steering angle and acceleration, which must remain within predefined physical
bounds. The environment is represented as a two-dimensional continuous
space discretised into a space-time grid.

 A safety zone is defined around each obstacle and agent to reduce the risk of collision and ensure safe navigation. Initially, this zone is represented as a circular region of radius $s$ for each obstacle irrespective of its size. When the obstacle begins to move, the safety zone morphs into an ellipse, with its major axis aligned with the direction of the obstacle's velocity and a focal distance of \(s^{+}=\frac{v_{o}}{2|b|}\), where $v_o$ represents obstacles velocity and $b$ is the comfort deceleration factor (see Fig.~\ref{safety}). The motion planning process identifies a path for the AV that avoids entering these safety zones, ensuring it steers clear of obstacles.

\begin{figure}[htbp]
\centerline{\includegraphics{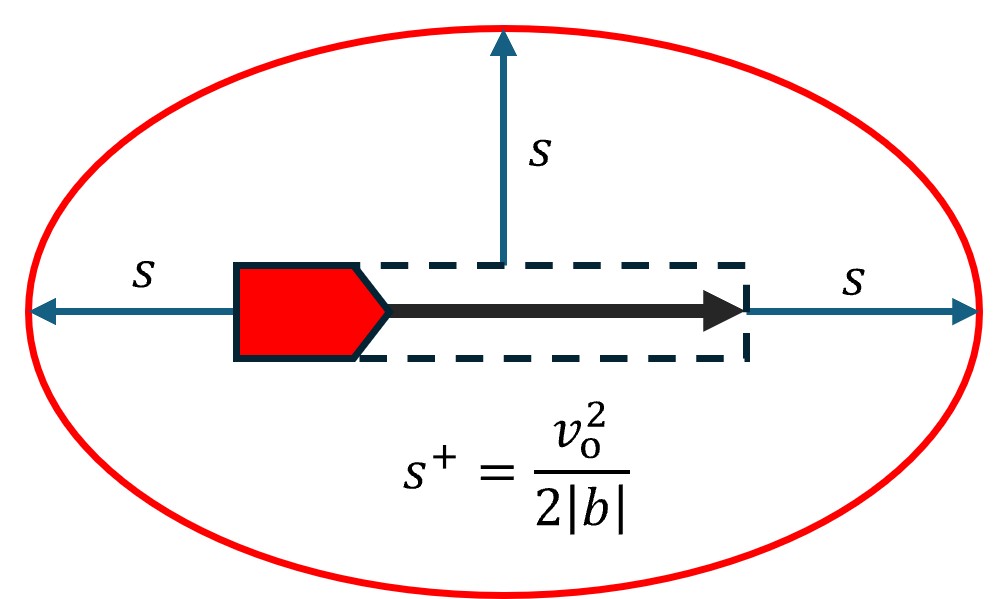}}
\caption{\textbf{The safety zone is an elliptical region surrounding a moving
obstacle}. The extended distance is calculated as the
obstacle's velocity (\(v_{o}\)) divided by the comfort
deceleration factor (\(b\)).}
\label{safety}
\end{figure}

\subsection*{Algorithmic framework}

The V* algorithm builds upon the classical A* framework, which searches
a graph \(G=(V,E)\) where $V$ is the set of nodes and $E$ is the set of edges representing feasible transitions between nodes. The goal is to find an optimal path from a start node \(s\) to a goal
node \(t\), with edges representing distances between vertices. Starting
at node \(s\), the algorithm incrementally explores \(G\) using a
branching operator and prioritizes nodes in an open list based on their
cost. The cost function evaluated at each node \(n\) is:
\begin{equation}
    f(n) = g(n) + h(n)
\end{equation}
where \(g(n)\) is the cost from the start node to \(n\), and \(h(n)\) is
the actual cost of the optimal path from \(n\) to the goal. While
\(g(n)\) is known during the algorithm's progress,
\(h(n\)) remains unknown, since the optimal path has not yet been
identified. Therefore, a heuristic estimate \(\widehat{h}(n)\) is used
to approximate \(h(n)\). The algorithm terminates at the goal node, and
the optimal path is reconstructed by backtracking from the goal to the
start using stored pointers.

However, A* does not account for realistic vehicle dynamics, as it
assumes transitions between grid points without considering velocity,
acceleration, or orientation constraints. To address this, the V*
algorithm extends A* by embedding motion dynamics directly into the
search structure. In V*, the search space is represented as a
discretised space-time graph, where each node captures the vehicle's
spatial position, heading angle, and speed at a given time step. The
algorithm initializes from a start node and generates successors by
discretising both direction and velocity. These successors are evaluated
based on actual travel time (reflecting the \(g\)-value) and an
admissible time-to-go heuristic (reflecting the \(\widehat{h}\)-value).
Fig.~\ref{spacetime} shows neighboring nodes of a given node at discrete time
intervals and spatial locations.

\begin{figure}[htbp]
\centerline{\includegraphics{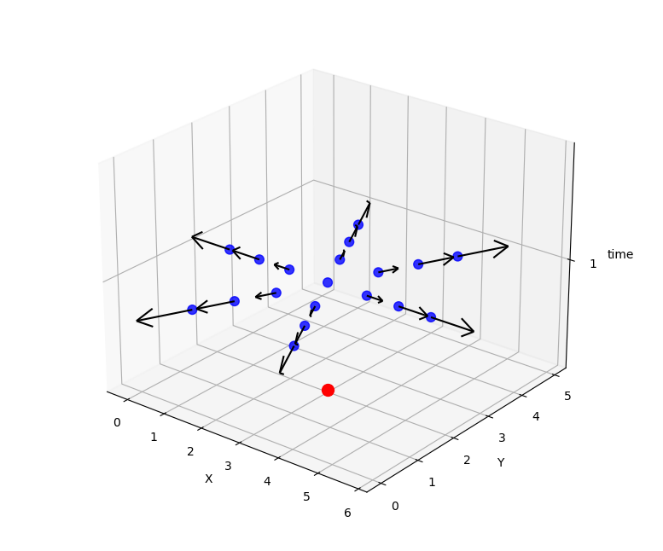}}
\caption{\textbf{Space-time representation of discretised motion planning nodes
centered around a current node with zero velocity (red) at time \(t=0\)}. Blue nodes are
the neighbouring nodes at \(t=1\), with outward-pointing arrows
illustrating heading and speed.}
\label{spacetime}
\end{figure}

The V* algorithm begins by adding the start node to the open list with a
cost \(\widehat{f}(n_0) = \widehat{h}\ (n_0)\), where \(\widehat{h}(n_0)\) is
the estimated remaining distance to the goal divided by the maximum
permissible speed. It then branches from the node $n$ with the smallest value of
\(\widehat{f}(n) = g(n) + \widehat{h}(n)\), generating successor nodes
by discretising directions and speeds into finite sets. Each neighbouring node $n^\prime$
represents a state defined by position and velocity, constrained by the
AV's dynamics (e.g., acceleration, deceleration, turning radius) using
the bicycle model \cite{mistri2019}. For each successor $n^\prime$ , the cost is
updated as \(g( n') = g(n) + 1\), with
\(\widehat{h}(n')\) estimating the remaining time from
\(n'\) to the goal. The algorithm terminates when it selects a node
\(n'\) within a suitably predefined small distance \(\varepsilon\) from the goal. Algorithm~\ref{alg:vstar} shows the detailed steps of the V* algorithm.

\begin{algorithm}[H]
\caption{V*} \label{alg:vstar}
\begin{flushleft}
\footnotesize   
\begin{algorithmic}[1]
\Require heading angle $\theta$, start position $n_0$, goal point $n_{\mathrm{goal}}$, termination threshold $\varepsilon$, wheelbase $L$, planning time $T$, 
         angular threshold $\delta_{\text{max}}$, max speed $v_{\mathrm{max}}$
\State \textbf{Initialize}
\State \textit{open\_list} $\gets$ a priority queue that sorts nodes based on their $f$ value in an ascending order
\State \textit{closed\_set} $\gets \emptyset$
\State $n \gets$ $n_0$
\State $n.g \,\gets\, 0$
\State $n.h \,\gets\, \hat{h}(n)$
\State \textit{open\_list} $\gets$ \textit{open\_list} $\cup \{n\}$
\While{\textit{open\_list} $\neq \emptyset$}
    \State $n \gets$ choose the node with minimum $f(n)$ within the \textit{open\_list} and remove it from the list
    \If{ the Euclidean distance between $n$ and $n_{\mathrm{goal}} < \varepsilon$}
        \State \Return path to $n$ by tracing back the parents
    \EndIf
    \State \textit{closed\_set} $\gets$ \textit{closed\_set} $\cup \{n\}$
    \State $N \gets$ \textbf{generate\_neighbours}($n$, $L$, $T$, $\delta_{\text{max}}$, $v_{\mathrm{max}}$)
    \ForAll{$n' \in N$}
        \If{$n' \notin$ \textit{closed\_set}}
            \State  $g \gets n.g + 1$
            \If{$n' \notin$ \textit{open\_list}}
                \State $n'.g  \gets\ g$
                \State $n'.h \gets \hat{h}(n')$
                \State $n'.f \gets n'.g + n'.h$
                \State $n'.\text{parent} \gets n$
                \State \textit{open\_list} $\gets$ \textit{open\_list} $\cup \{n'\}$
            \ElsIf{ $g<\, n'.g$}
                \State $n'.g \gets g $
                \State $n'.f \gets n'.g + n'.h$
                \State $n'.\text{parent} \,\gets\, n$
                \State Update $n'$ in \textit{open\_list} and re‐sort the list by increasing $n.f$
            \EndIf
        \EndIf
    \EndFor
\EndWhile
\end{algorithmic}

\vspace{0.5ex}
\textbf{Note:} The “$n.*$” notation indicates a property (field) of node $n$.
\end{flushleft}
\end{algorithm}

To construct the set of dynamically feasible neighbor nodes, the
function \textbf{generate\_neighbours}() defines a localized hexagonal
lattice centered at the vehicle's current position.
Then, the candidate set is filtered by applying geometric and kinematic
constraints, ensuring that only transitions consistent with the
vehicle's orientation and motion model are retained.

The procedure begins by generating a bounded region of lattice offsets,
determined by the maximum admissible speed. For each surviving candidate, the function computes
the turning angle required to reach the new node, from which the turning
radius and angular velocity are inferred. Nodes are admitted only if
their turning radius falls within the feasible bounds of the vehicle's
steering system and if the resulting angular velocity remains within
limits, preventing high-speed or discontinuous maneuvers. The final
output is a refined list of reachable nodes, each encoded with position,
heading, and path length. Algorithm~\ref{alg:generate_neighbours} shows the detailed steps of node
generation function.

\begin{algorithm}[H]
\caption{generate\_neighbours}
\label{alg:generate_neighbours}
\begin{algorithmic}[1]
\Require Current state $(x, y, \theta,v,t)$, lattice magnitude $m$, wheelbase $L$, planning time $T$, 
         angular threshold $\delta_{\text{max}}$, max speed $v_{\mathrm{max}}$, maximum acceleration $a$ 
\State \textbf{Initialize:} 
\State $\mathbf{u_1} \gets [m, 0]$, $\mathbf{u_2} \gets [\frac{m}{2}, \frac{m\sqrt{3}}{2}]$ \Comment{Hexagonal base vectors}
\State $\mathbf{M} \gets [\mathbf{u_1} | \mathbf{u_2}]$ \Comment{Lattice matrix (2×2)}
\State $I_{\text{max}} \gets \lceil v_{\text{max}} / m \rceil + 1$, $N \gets \emptyset$
\For{each lattice index pair $(i,j)$ in $[-I_{\text{max}}, I_{\text{max}}]^2 \setminus \{(0,0)\}$}
    \State $\mathbf{p} \gets \mathbf{M} \cdot [i, j]^\mathrm{T}$ \Comment{Generate hexagonal lattice point}
    \State $dx, dy \gets \mathbf{p}[0], \mathbf{p}[1]$
    \State $\theta'\gets\text{the heading angle of }\mathbf{p}$
    \If{  $ (dx \cdot \cos\theta + dy \cdot \sin\theta) / \|\mathbf{p}\| \le \cos\delta_{\mathrm{max}}$ }
        \State $(x', y') \gets (x, y) + \mathbf{p}$
        \State $\phi \gets \text{ get the angular difference between } \theta' \text{ and } \theta$
        \State $r \gets$ turning radius from arc geometry using $\phi$ and $\mathbf{p}$
        \State $v' \gets$ compute the speed from $r$, $\phi$, and $T$
        \State $t'\gets t+1$
        \If{ $0\le v'$ \textbf{and } $ v-a \le v'$ \textbf{and} $v'\le v+a$ \textbf{and} $v'\le v_{\mathrm{max}}$ \textbf{and }  $r > L/m$ }
            \State $N \gets N \cup \{[x', y', \theta', v',t']\}$
        \EndIf
    \EndIf
\EndFor
\State \Return $N$
\end{algorithmic}

\end{algorithm}

\subsection*{Theoretical Analysis}

To rigorously analyze the performance of the V* algorithm, this section
presents a series of theoretical propositions, each accompanied by a
 proof. These propositions are chosen to address key aspects of
the algorithm's characteristics: Optimality, spatial dispersion of waypoints,
efficient coverage of the search space, and guaranteed termination.
Proposition~\ref{pr_optimal} establishes that V* yields optimal solutions under a
consistent and admissible heuristic. Propositions~\ref{pr_distabce} and~\ref{pr_minNumber} provide
geometric foundations for designing an efficient graph structure by
determining the optimal heading increment that balances coverage with
sparsity. Proposition~\ref{pr_terminate} guarantees algorithmic completeness by ensuring
termination under reasonable assumptions.
\begin{proposition}[Optimality of V*]
\label{pr_optimal}
If $\widehat{h}(n) = \frac{d(n,n_{\mathrm{goal}})}{v_{\max}}$ where $d(n,n_{\mathrm{goal}})$ is the Euclidean distance from node $n$ to the goal, then the V* algorithm guarantees finding an optimal path within the discretised space-time-velocity lattice.
\end{proposition}

\begin{proof}
We establish optimality by proving the heuristic is both admissible and consistent, then proceed to prove the proposition by contradiction, assuming that V* terminates with a suboptimal path to the goal.

\textbf{Admissibility:} The heuristic $\widehat{h}(n) = \frac{d(n,n_{\mathrm{goal}})}{v_{\max}}$ represents the theoretical minimum time to reach the goal, assuming direct straight-line motion at maximum velocity. Since no feasible path can traverse the Euclidean distance faster than this physical limit, we have $\widehat{h}(n) \leq h(n)$ for all nodes $n$, where $h(n)$ is the true optimal cost-to-go in the lattice.

\textbf{Consistency:} For any two adjacent nodes $n, n'$ in the V* lattice (connected by one time step $\tau$), we must show $\widehat{h}(n) \leq \tau + \widehat{h}(n')$.

Since $n'$ is reachable from $n$ in one time step, the maximum spatial displacement is bounded by $v_{\max} \cdot \tau$. By the triangle inequality:
$$d(n,n_{\mathrm{goal}}) \leq d(n,n') + d(n',n_{\mathrm{goal}}) \leq v_{\max} \cdot \tau + d(n',n_{\mathrm{goal}})$$

Dividing by $v_{\max}$:
$$\frac{d(n,n_{\mathrm{goal}})}{v_{\max}} \leq \tau + \frac{d(n',n_{\mathrm{goal}})}{v_{\max}}$$

Therefore, $\widehat{h}(n) \leq \tau + \widehat{h}(n')$, establishing consistency.

\textbf{Optimality:} Suppose V* terminates at goal node $\tilde{n}_{\mathrm{goal}}$ with cost $g(\tilde{n}_{\mathrm{goal}})$, and assume this is suboptimal compared to an optimal goal node $n_{\mathrm{goal}}$ with cost $g^* < g(\tilde{n}_{\mathrm{goal}})$.

At termination, there must exist some node $n'$ on the optimal path that remains in the open list. For this node:
$$f(n') = g(n') + \widehat{h}(n') \leq g(n') + h(n') = g^* < g(\tilde{n}_{\mathrm{goal}}) = f(\tilde{n}_{\mathrm{goal}})$$

Since V* selects the node with minimum $f$-value, it would have selected $n'$ before $\tilde{n}_{\mathrm{goal}}$, contradicting our assumption. Therefore, V* finds an optimal path within the discrete lattice.
\end{proof}
While this proposition establishes optimality within the discrete lattice, the relationship between discrete and continuous optimality depends on the lattice resolution and the specific kinematic constraints. As the spatial and temporal discretisation becomes finer, the V* solution approaches the continuous optimal solution.
To leverage the properties of the V* algorithm, the discretised
directions should sufficiently cover the space while minimizing the
generation of points that are excessively close to one another. If a
point is determined to be unsafe due to collision risk or path
infeasibility, exploring its neighboring points becomes unnecessary.
Therefore, to mitigate the state-space explosion, the network generation
must ensure that adjacent points maintain a minimum separation equal to
the specified safe distance.

\begin{proposition}
 \label{pr_distabce}
    For a network of waypoints generated by moving a fixed distance \(\varepsilon\) in a fixed direction from a starting point
\(\left( x_{0},y_{0} \right)\), the minimum distance between any two distinct waypoints is \(\varepsilon\), if \(\theta = \frac{\pi}{3}\).
\end{proposition}

\begin{proof}
For any value of \(\theta\), not an integer multiple of
\(\pi\), let \(\left( x_{i},y_{i} \right)\) represent the coordinates of
the $i$-th point in the network, where the points are generated
iteratively. Starting from the starting point
\(\left( x_{0},y_{0} \right)\) the positions are determined using:
\begin{equation}
    x_{i} = x_{i - 1} + \varepsilon\cos(i\theta),\quad y_{i} = y_{i - 1} + \varepsilon\sin(i\theta)
\end{equation}

\noindent This implies that the points are formed periodically by two basis
vectors:
\begin{equation}
    \mathbf{u}_{\mathbf{1}} = (\varepsilon,0),\quad\mathbf{u}_{\mathbf{2}} = \left( \varepsilon\cos\theta,\varepsilon\sin\theta \right)
\end{equation}
Each point in the network can thus be expressed as:
\begin{equation}
    \left( x_{i},y_{i} \right) = \alpha_i\mathbf{u}_{\mathbf{1}} + \beta_i\mathbf{u}_{\mathbf{2}}
\end{equation}
where \(\alpha_i,\beta_i \in \mathbb{Z}\). 
The Euclidean distance between two distinct \(( x_{i},y_{i})\) and \(( x_{j},y_{j})\) , with \( i \neq j\) , is given by:

\begin{equation}
    \eta = \sqrt{\left( x_{i} - x_{j} \right)^{2} + \left( y_{i} - y_{j} \right)^{2}}
\end{equation}

\noindent The vector between these points is a linear combination of the basis
vectors:
\begin{equation}
    \left( x_{i} - x_{j},y_{i} - y_{j} \right) = m_{1}\mathbf{u}_{\mathbf{1}} + m_{2}\mathbf{u}_{\mathbf{2}}
\end{equation}

\noindent where \(m_{1},m_{2} \in \mathbb{Z} / \{0\} \).
Substituting the basis vectors, the distance becomes:

\begin{equation}
    \eta\, = \, \varepsilon\,\sqrt{\left( m_{1}\, + \, m_{2}\,\cos\theta \right)^{2}\, + \,\left( m_{2}\,\sin\theta \right)^{2}}
\end{equation}

\noindent To minimize this distance the smallest possible values of \(m_{1}\) and
\(m_{2}\) are \(\pm 1\) , corresponding to the nearest neighbors.

Case 1: \(m_{1} = \pm 1,m_{2} = 0\):
\begin{equation}
    \eta = \varepsilon\sqrt{(1 + 0)^{2} + 0^{2}} = \varepsilon
\end{equation}

Case 2: \(m_{1} = 0,m_{2} = \pm 1\):
\begin{equation}
    \eta = \varepsilon\sqrt{\left( 0 + \cos\theta \right)^{2} + \sin^2\theta}
\end{equation}
Using the trigonometric identity
\(\cos^{2}(\theta) + \sin^{2}(\theta) = 1\), this results \(\varepsilon\).

Case 3: \(m_{1} = \pm 1,m_{2} = \pm 1\):
\begin{align}
     &\varepsilon\,\sqrt{m_{1}^{2}\, + 2m_{1}m_{2}\,\cos\theta + m_{2}^{2}\cos^{2}\theta\, + \, m_{2}^{2}\,\sin^{2}\theta} = \\&\varepsilon\,\sqrt{m_{1}^{2}\, + 2m_{1}\, m_{2}\,\cos\theta + m_{2}^{2}}
\end{align}
Since \(m_{1} = \pm 1,m_{2} = \pm 1\), then we have:

\begin{equation}
    \eta = \varepsilon\,\sqrt{2\, \pm \, 2\cos\theta}
\end{equation}
The minimum distance in this case is
\(\eta = \varepsilon\,\sqrt{2\, - 2\cos\theta}\), and since \(\cos\left( \frac{\pi}{3} \right) = \frac{1}{2}\) , \(\eta\) equals to \(\varepsilon\). 

\end{proof}

\begin{proposition}
 \label{pr_minNumber}
    The minimum number of nodes is generated to explore a given space when \(\theta\  = \frac{\pi}{3}\).
\end{proposition}

\begin{proof}
    We examine two possible cases for different values of \(\theta\).
    
    Case 1: \(\theta <\) \(\frac{\pi}{3}\):
    
    At each node expansion step, when \(\theta\  = \frac{\pi}{3}\), six
    nodes are generated in the graph to cover the adjacent space. Any value
    of \(\theta\) less than \(\frac{\pi}{3}\) results in the generation of
    more nodes.
    
    Case 2: \(\theta >\) \(\frac{\pi}{3}\):
    
    The efficiency of the convex hull for a set of discrete points is
    evaluated based on the density of coverage. Each convex hull is bounded
    by a hexagon formed by connecting neighboring points, and the area of
    the hexagonal unit cell is calculated as:
    \begin{align}
        A_\text{hex} &= \frac{2\pi}{\theta} \parallel \mathbf{u}_{\mathbf{1}} \parallel   \parallel \mathbf{u}_{\mathbf{2}} \parallel  \sin\theta \\ &= \varepsilon^{2} \sin\left( \frac{2\pi}{3} \right)6 
        = \varepsilon^{2}3\sqrt{3}
    \end{align}  
    For any value of \(\theta\) greater than \(\frac{\pi}{3}\), the ratio
    \(\frac{2\pi}{\theta}\) decreases faster than \(\sin\theta\)
    increases, resulting in reduced area coverage. This implies that either
    more convex hulls are required to cover the given space, or the covered
    area is proportional to the hexagonal convex hull. This demonstrates
    that when \(\theta = \frac{\pi}{3}\), the minimum number of nodes is
    generated to effectively explore the given space.  
\end{proof} 

Fig.~\ref{node}
illustrates that smaller angles result in overlapping nodes and
increased redundancy due to overly dense node generation, leading to
inefficient use of computational resources. Conversely, larger angles
reduce redundancy but sacrifice spatial coverage, creating gaps in the
network that can compromise the algorithm's efficiency.

\begin{figure}[ht]
\centerline{\includegraphics[width=\columnwidth]{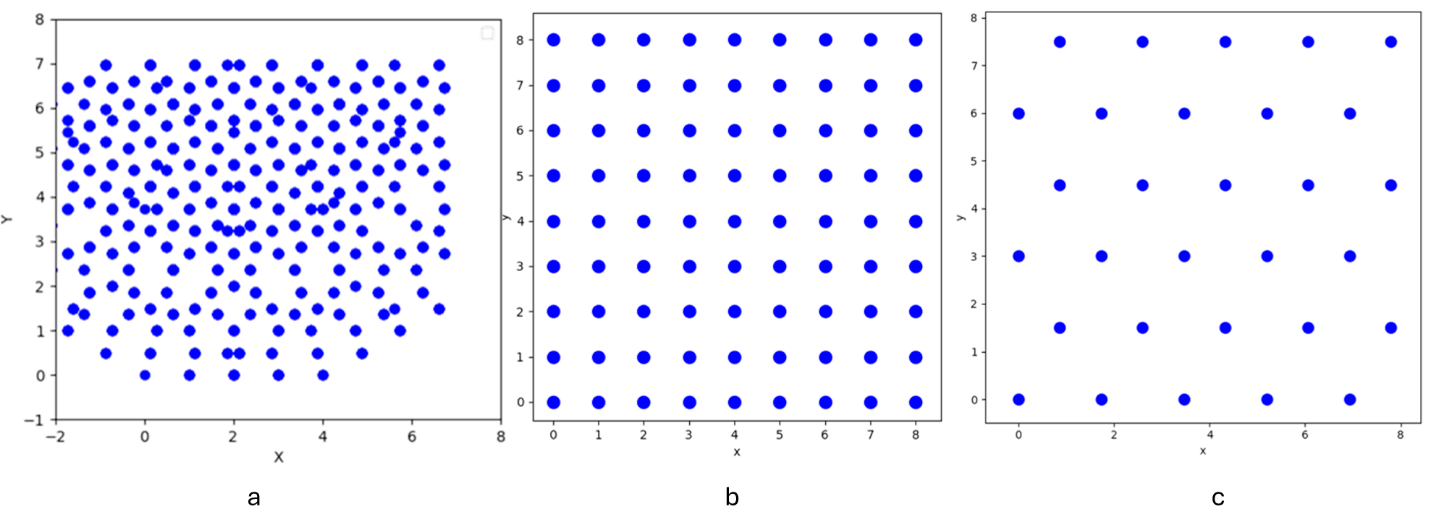}}
\caption{\textbf{Node distributions for networks generated in the
$x$-$y$ plane with different angles \(\theta\)}. Panel (a) shows
the distribution for \(\theta = \frac{\pi}{6},\) resulting in
overlapping nodes and increased redundancy. Panel (b) represents
\(\theta = \frac{\pi}{2}\) needing more nodes to cover the same area.
Panel (c) illustrates the optimal distribution at
\(\theta = \frac{\pi}{3}\), where the minimum number of nodes is
generated, maintaining a uniform distance between distinct points.}
\label{node}
\end{figure}

\begin{proposition}[Termination of V*]
\label{pr_terminate}

The V* algorithm with stationary waiting actions terminates within finite time provided that:
\begin{itemize}
    \item The space-time-velocity lattice is bounded by maximum planning horizon $T_{\max} < \infty$
    \item Spatial and velocity discretizations are finite
    \item The heuristic is admissible and consistent
    \item Edge costs are bounded: $c(n,n') = \tau \geq \epsilon > 0$ for all transitions
\end{itemize}
\end{proposition}

\begin{proof}
We establish termination by proving the search space is finite and the algorithm cannot loop indefinitely, accounting for both motion and stationary states.

\textbf{Finite State Space:} The total number of possible states in the V* lattice is bounded by $\S = |X| \times |Y| \times |\Theta| \times |V| \times |T| < \infty$ where each dimension is finite by assumption. Specifically, spatial positions are bounded by $|X| \times |Y| = O((\frac{R}{\varepsilon})^2)$ where $R$ is the maximum reachable distance, headings by $|\Theta| = \frac{2\pi}{ \Delta \theta}$, velocities by $|V| = \frac{v_{\max}}{\Delta v} + 1$, and time steps by $|T| = \frac{T_{\max}}{\tau}$. Crucially, even when V* generates stationary states for waiting behaviors, each combination $(x, y, \theta, v, t)$ represents a unique state in the space-time lattice. Multiple stationary nodes at the same spatial location with different arrival times $(x, y, \theta, 0, t_1), (x, y, \theta, 0, t_2), \ldots$ are distinct states, and their total number remains bounded by the finite time horizon.

\textbf{Bounded Stationary Actions:} For any spatial location, the algorithm can generate at most $\frac{T_{\max}}{\tau}$ distinct waiting states, corresponding to arrival times $t, t+\tau, t+2\tau, \ldots, T_{\max}$. Since the total number of spatial locations is finite and the time horizon is bounded, the number of possible stationary states is bounded by $|X| \times |Y| \times |\Theta| \times \frac{T_{\max}}{\tau}$. The algorithm cannot wait indefinitely at any location because each waiting action incurs cost $\tau$, making excessive waiting suboptimal compared to motion toward the goal.

\textbf{Monotonic Cost Property:} Let $g^*$ be the cost of an optimal path to the goal, which is finite by the reachability assumption. Since the heuristic is admissible, we have $f(n_0) = \widehat{h}(n_0) \leq g^*$. For any node $n$ with $f(n) > g^*$, including both motion and stationary nodes, if $n$ were expanded, any path through $n$ would have cost at least $f(n) > g^*$, making it suboptimal. Since V* maintains the property of expanding nodes in order of increasing $f$-value, it will find a goal node with cost $g^*$ before expanding any node with $f(n) > g^*$.

\textbf{Bounded Expansion:} The number of nodes with $f(n) \leq g^*$ is finite because time is bounded by $t \leq T_{\max}$, spatial reach is bounded by maximum reachable distance from the start, and all state dimensions are discrete and finite. This bound applies to both motion states and stationary waiting states, as waiting beyond the optimal solution cost becomes dominated by nodes with better $f$-values.

\textbf{Goal Convergence:} When V* expands a node within the termination threshold $\varepsilon$ of the spatial goal location, it terminates successfully. Since the minimum distance between distinct lattice nodes is bounded by the discretization resolution, the algorithm will eventually reach the goal region through either motion or optimally-timed waiting behaviors.

\textbf{No Infinite Cycles:} V* maintains a closed set of expanded nodes and never re-expands them. Combined with the finite state space that includes all possible stationary states, this prevents infinite loops. The algorithm cannot cycle between waiting states at the same location because each waiting state has a distinct time component, and nodes are only expanded once.

Therefore, V* must terminate by either finding a goal node through successful termination, or exhausting all nodes with $f(n) \leq g^*$ without finding a goal, indicating no feasible solution exists within the bounded lattice. Both cases result in finite termination regardless of whether the optimal solution involves motion, waiting, or a combination of both behaviors.
\end{proof}

\subsection*{Kinematic Constraints}
The V* algorithm generates discrete waypoints that represent vehicle
states in space and time. However, when these points deviate in
direction, the resulting path implicitly assumes that the vehicle can
execute abrupt turns between successive waypoints. In practice, such
instantaneous changes in orientation are not physically feasible, as
vehicles follow smooth, curved trajectories.

To ensure realistic motion representation, V* uses a simplified version
of the Ackermann steering system \cite{simionescu2002}. This system governs the steering
geometry of vehicles, ensuring that the wheels trace concentric arcs
during a turn and avoid lateral slipping. For modeling purposes, it is
commonly reduced to the bicycle model \cite{mistri2019}, in which a single
front and rear wheel represent the front and rear axles of the vehicle,
respectively. The front wheel is used to steer the vehicle, and its
motion determines the turning behavior. The state of the vehicle in the
bicycle model is defined by its position (\(x,y\)), heading angle
\(\theta\), rear wheel speed \(v\), and steering angle \(\delta\). The
equations of motion for the bicycle model are:

\begin{align}
    \dot{x} &= v\cos\theta \\
    \dot{y} &= v\sin\theta \\
    \label{thetadot}
    \dot{\theta} &= \frac{v}{L}\tan\delta \\
    \dot{v}  &=  a
\end{align}
where \(L\) is the distance between the front and rear axles, \(a\) is the scalar acceleration  of the rear wheel. Fig.~\ref{bicyclemodel} shows a schematic representation of the bicycle model used in this paper.

\begin{figure}[htbp]
\centerline{\includegraphics[width=.5\columnwidth]{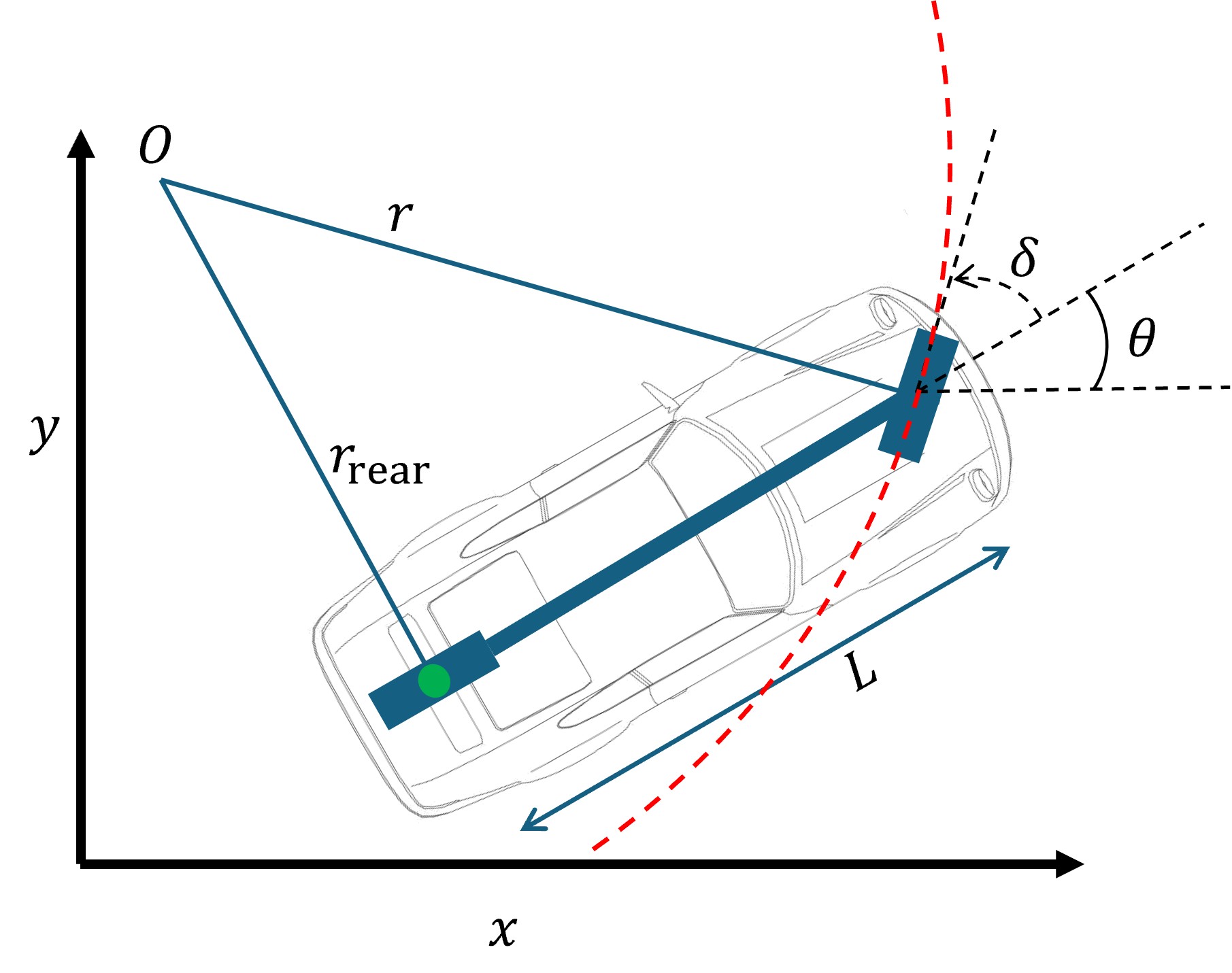}}
\caption{\textbf{Schematic representation of the Ackermann steering mechanism.}
The front wheel is following a curved trajectory (dashed red line), with
a steering angle \(\delta\). The radius of the circular path traced by the
front is denoted by \(r\), and \(r_{\text{rear}}\) shows the rear
axle turning radius.}
\label{bicyclemodel}
\end{figure}

When a vehicle initiates a turn from a straight-line trajectory, the
velocity of the rear wheels is no longer the same as front wheel.
Initially, the rear wheel begins to lag, and its speed slows down over
time until it eventually converges to a fixed value. As the turn develops, the rear wheel's speed decreases gradually and
stabilizes at a steady-state. Fig.~\ref{steadystate} shows the trajectory of a front
and rear wheel of a turning vehicle, since the beginning of the turning
until the reaching to the steady state.

\begin{figure}[htbp]
\centerline{\includegraphics{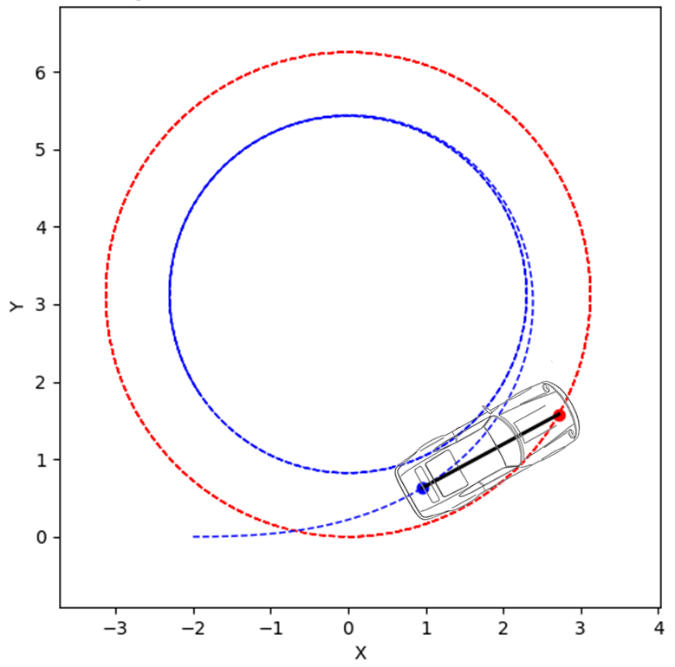}}
\caption{\textbf{The front wheel and rear wheel trajectory from the initial
state until the steady-state.}}
\label{steadystate}
\end{figure}

When turning, the front wheel follows a circular path with radius \(r\)
at an angular speed \(\omega\). At the same time, the vehicle's geometry
fixes the distance between the front and rear wheels (the wheelbase) to
be \(L\). In the steady-state, the rear wheel travels along a circle
whose center is the same as the turning center of the front wheel, but with a different radius, \(r_{\text{rear}}\).  Therefore, by the Pythagorean theorem, we have:

\begin{equation}
    r^2 = L^2 + r_{\text{rear}}^2
\end{equation}

Since the AV rotates as a rigid body with the same angular speed \(\omega\),
the rear wheel's linear speed is simply the product of the angular speed and the radius of its path. Therefore:

\begin{equation}
\label{eq:vfinal}
    v_{\infty}  = \omega r_{\text{rear}}\  = \omega\sqrt{r^{2}-\ L^{2}}
\end{equation}

We model the rear wheel speed, \(v(t)\), as being determined by the
cosine of the instantaneous steering angle \(\delta(t)\),

\begin{equation}
    v(t) = r\omega\cos\left\lbrack \delta(t) \right\rbrack
\end{equation}
where \(r\) is the front wheel's circular path radius. At \(t\  = \ 0\),
when the vehicle travels in a straight line (\(\delta(0) = 0\)), we
have \(v_{r}(0)= r\omega\). In steady state, the rear wheel speed
becomes

\begin{equation}
    v_{\infty}  =  r\omega\cos\delta_{\infty}.
\end{equation}

\noindent According to Eq.~\ref{eq:vfinal}, we have \(\cos\delta_{\infty} = \frac{\ \sqrt{r^{2}-\ L^{2}}}{r}\). 

\noindent To capture the transition between these two limits, we assume that the
steering angle evolves exponentially toward its final value. That is, we
set:

\begin{equation}
\label{eq:delta}
    \delta(t)  = \delta_{\infty}\ \lbrack 1 -\exp{(-\beta t)}\rbrack
\end{equation}
where \(\beta\  > \ 0\) is a decay constant governing the rate at which
the steering angle approaches \(\delta_{\infty}\). Substituting this into
the expression for \(v(t)\) yields

\begin{equation}
    v(t) = \ r\omega\cos\left( \delta_{\infty}\left\lbrack 1\ -\exp(-\beta t) \right\rbrack \right).
\end{equation}

\noindent  This formulation satisfies the correct boundary conditions: at
\(t  =  0\), \(\exp{(-\beta \cdot 0)}  =  1\), so
\(\delta(0)  =  0\) and \(v(0)  =  r\omega\); as
\(t \rightarrow \infty,\ \exp{(-\beta t)}\  \rightarrow \ 0\),
so \(\delta(\infty)  = \delta_{\infty}\) and
\(v(\infty)\  = r\ \omega\cos{(\delta_{\infty})}\  = \ \omega\sqrt{r^{2}-\ L^{2}}\).

\noindent In the kinematic bicycle model, the rate of change of the body
orientation is related to the vehicle's speed and the steering
angle. Specifically, one has

\begin{equation}
    \frac{d\theta}{dt} = \left( \frac{v(t)}{L} \right)\tan\delta(t)
\end{equation}

\noindent To model the transition in steering, we assume that the steering angle increases
exponentially from zero to a final value \(\delta_{\infty}\) as represented in Eq.~\ref{eq:delta}.
\noindent  Thus, the orientation equation becomes
\begin{align}
     v(t) &= \ r\omega\cos\left( \delta_{\infty} \ \left\lbrack 1\ -\exp(-\beta t) \right\rbrack \right) \\
     \frac{d\theta}{dt} &= \left( \frac{v(t)}{L} \right) \tan\left( \delta_{\infty}\left\lbrack 1\ -\exp(-\beta t) \right\rbrack \right)
\end{align}

\noindent  Therefore,
\begin{equation}
    \theta(t) = \ \left( \frac{r\omega}{L} \right)\int_{s = 0}^{t}{\sin\left( \delta_{\infty} \left\lbrack 1-\exp(-\beta s) \right\rbrack \right)ds}
\end{equation}
The transient behavior is captured by the exponential term
\(\exp( -\beta t)\), where \(\beta\) functions as a time constant that
governs the rate of convergence toward the steady-state condition and
has units of 1/s. The value of \(\beta\) is
determined by the interaction of several geometric and kinematic
parameters, which influence how rapidly the vehicle adjusts during
turning maneuvers.

When a vehicle initiates a turn, it does so at an angular speed
\(\omega\), and this parameter has a direct bearing on the rate of
transition. A higher \(\omega\) implies a sharper turn and a greater
demand on the system to reorient quickly. As such, \(\beta\) must scale
accordingly to reflect the faster dynamics. The radius \(r\) of the
front wheel's path, indicative of curvature, also plays a role. A larger
\(r\) corresponds to a gentler curve, leading to a smaller discrepancy
in the alignment between the front and rear wheels, and thereby
permitting a more rapid transition. Hence, \(\beta\) is expected to
increase with \(r\) as well.

The wheelbase \(L\), defined as the distance between the front and rear
axles, introduces a spatial lag in the motion of the rear wheel. A
larger \(L\) causes the vehicle to turn more gradually because the rate
of change in heading angle decreases as \(L\) increases. This is evident
from the kinematic Eq.~\ref{thetadot}. As \(L\) increases, the angular
velocity becomes smaller, meaning the vehicle's orientation changes more
slowly. Consequently, the rear wheel takes longer to align with the new
direction, resulting in a slower convergence of its speed. This
indicates that \(\beta\) is inversely proportional to \(L\). Finally, we
can express these relationships as:

\begin{equation}
    \beta\  \propto \frac{\omega r}{L} \rightarrow \beta = k\frac{\omega r}{L}
\end{equation}

\noindent where \(k\) is a dimensionless parameter calibrated through empirical
observation. Numerical simulations suggest that for a wheelbase of \qty[mode=text]{2}{\meter},
\(k\ \)consistently equals 0.4 regardless of the values of \(\omega\)
and \(r\).

Once the body orientation \(\theta(t)\) is known, the rear wheel's
position can be obtained using simple geometry. If the front wheel's
coordinates are given by \(F_{x}(t)\) and \(F_{y}(t)\), then the rear
wheel position is computed as:
\begin{align}
    R_{x}(t) &= F_{x}(t) - L\cos\theta(t) \\
    R_{y}(t) &= F_{y}(t) - L\sin\theta(t)
\end{align}

\section*{Discussion}

This study introduces V*, a velocity-aware motion planning algorithm that explicitly incorporates velocity and time as structural parameters within the search graph. While building upon established kinodynamic planning principles, V* offers a distinct approach by encoding velocity as node state variables rather than embedding velocity information within motion primitives or handling it through post-processing steps.

The core innovation lies in V*'s dynamic graph generation approach, where feasible transitions are verified on-demand through geometric pruning and real-time kinematic bicycle model checking. This enables discovery of motion combinations not explicitly precomputed, providing greater adaptability compared to Hybrid A* and similar primitive-based approaches that rely on offline boundary value problem solvers. Our theoretical contributions establish that the hexagonal discretisation strategy achieves minimal node redundancy while maintaining uniform waypoint spacing, and we provide formal termination guarantees even when incorporating stationary waiting behaviors.

The mathematical formulation for transient steering dynamics captures realistic vehicle maneuvering behavior, ensuring that discrete waypoint sequences translate to physically feasible continuous trajectories without requiring post-hoc trajectory refinement. Experimental evaluation demonstrates substantial performance improvements through the waterflow heuristic, which reduced node expansions by up to 98\% while maintaining optimal solution quality by incorporating environmental obstacles into cost-to-go estimation.

V* provides fundamental advantages in temporal reasoning through its native time-velocity representation. The explicit incorporation of zero-velocity states enables direct representation of waiting behaviors, allowing proactive coordination with moving obstacles. The dynamic obstacle simulations illustrate this capability, where the vehicle demonstrates anticipatory waiting to yield to crossing obstacles. The dimensional expansion to five-dimensional search space represents a trade-off between expressiveness and computational complexity, but geometric pruning strategies effectively manage computational overhead by eliminating infeasible transitions during node expansion.

The algorithm enables fundamentally new planning capabilities that emerge naturally from its representation: temporal coordination with moving obstacles, intentional waiting behaviors, and direct velocity-time trade-off optimization. Unlike primitive-based approaches that require specialized design, V*'s native space-time-velocity representation makes these capabilities architecturally inherent, positioning it as a foundational approach for autonomous vehicle motion planning with practical feasibility for real-world deployment.

Although the current formulation addresses only deterministic single-agent scenarios, the structure of V* provides
a solid foundation for broader applications. Its velocity-aware representation naturally extends
to multi-agent settings, where coordination depends on both spatial positioning and temporal consistency.
Additionally, the heuristic function could be adapted to incorporate probabilistic models or
data-driven priors, enabling the algorithm to handle uncertainty and dynamic environments. These
characteristics position V* as a flexible framework capable of supporting more complex planning
tasks beyond its current scope.

\clearpage 

%
\bibliography{science_template} 
\bibliographystyle{sciencemag}

%
%
%
%
%
%


\section*{Acknowledgments}
The authors gratefully acknowledge the University of Sydney for providing research facilities and support throughout this project. Special thanks are extended to Bita Ghamooshi Ramandi and Ali Afrasiabi for their insightful feedback and valuable assistance to this work.

\paragraph*{Funding}
This work was supported by the Australian Research Council Discovery Project~DP220100882.

\paragraph*{Author contributions}
\textbf{A.Z.A.}: Conceptualization; Methodology; Software; Writing – original draft.  
\textbf{M.G.H.B.}: Supervision; Methodology; Writing – review \& editing.  
\textbf{M.R.}: Supervision; Methodology; Writing – review \& editing.  
\textbf{G.G.}: Advising; Writing – review \& editing.

\end{document}